\documentclass{article}

\PassOptionsToPackage{numbers, compress}{natbib}
\usepackage[final]{nips_2017}

\usepackage[utf8]{inputenc} % allow utf-8 input
\usepackage[T1]{fontenc}    % use 8-bit T1 fonts
\usepackage{hyperref}       % hyperlinks
\usepackage{url}            % simple URL typesetting
\usepackage{booktabs}       % professional-quality tables
\usepackage{amsfonts}       % blackboard math symbols
\usepackage{nicefrac}       % compact symbols for 1/2, etc.
\usepackage{microtype}      % microtypography
\usepackage{amsmath}

\usepackage{algorithm2e}
\usepackage{graphicx}
\usepackage{multirow}

\usepackage{amsthm}
\newtheorem{theorem}{Theorem}
\newtheorem{lemma}{Lemma}

\DeclareMathOperator*{\argmax}{arg\,max}
\newcommand{\Lb}{\mathbb{L}}
\newcommand{\Rb}{\mathbb{R}}
\newcommand{\Hb}{\mathbb{H}}
\newcommand{\Bb}{\mathbb{B}}

\title{Large-Margin Classification in Hyperbolic Space}
		
\author{
Hyunghoon Cho\\
  Computer Science and Artificial Intelligence Laboratory\\
  Massachusetts Institute of Technology\\
  Cambridge, MA 02139 \\
  \texttt{hhcho@mit.edu} \\
  \AND
Benjamin DeMeo\\
Department of Biomedical Informatics \\
Harvard University\\
Cambridge, MA 02138 \\
\texttt{bdemeo@g.harvard.edu} \\
   \AND
   Jian Peng \\
   Department of Computer Science \\
   University of Illinois at Urbana-Champaign \\
   Urbana, IL 61801 \\
   \texttt{jianpeng@illinois.edu} \\
   \AND
   Bonnie Berger \\
   Computer Science and Artificial Intelligence Laboratory\\
   Massachusetts Institute of Technology\\
   Cambridge, MA 02139 \\
   \texttt{bab@mit.edu}
}

\begin{document}

\maketitle

\begin{abstract}
Representing data in hyperbolic space can effectively capture latent hierarchical relationships. With the goal of enabling accurate classification of points in hyperbolic space while respecting their hyperbolic geometry, we introduce hyperbolic SVM\footnote{
A MATLAB implementation of hyperbolic SVM and our benchmark datasets are provided at: \url{https://github.com/hhcho/hyplinear}.}, a hyperbolic formulation of support vector machine classifiers, 
and elucidate 
through new theoretical work
its connection to the Euclidean counterpart. We demonstrate the performance improvement of hyperbolic SVM for multi-class prediction tasks on real-world complex networks as well as simulated datasets.
Our work allows
analytic pipelines that take the inherent hyperbolic geometry of the data into account in an end-to-end fashion without resorting to ill-fitting tools developed for Euclidean space.
\end{abstract}

%\begin{nolinenumbers}
\section{Introduction}
%\end{nolinenumbers}

Learning informative feature representations of symbolic data, such as text documents or graphs, is a key factor determining the success of downstream pattern recognition tasks.
Recently, embedding data into hyperbolic space---a class of non-Euclidean spaces with constant negative curvature---has been receiving increasing attention due to its effectiveness in capturing latent hierarchical structure~\cite{AlanisLobato16,Chamberlain17,DeSa18,Krioukov10,Nickel17,Papadopoulos15}.
This capability is likely due to the key property of hyperbolic space that the amount of space grows \emph{exponentially} with the distance from a reference point, in contrast to the slower, polynomial growth in Euclidean space.
The geometry of tree-structured data, which similarly expands exponentially with distance from the root, can thus be accurately captured in hyperbolic space, but not in Euclidean space~\cite{Krioukov10}.

Motivated by this observation, a number of recent studies have focused on developing effective algorithms for embedding data in hyperbolic space in 
various domains, including natural language processing~\cite{DeSa18,Nickel17} and network science~\cite{AlanisLobato16,Chamberlain17,Papadopoulos15}. 
Using the hyperbolic embeddings with only a small number of dimensions, these methods were able to achieve superior performance in their respective downstream tasks (e.g., answering semantic queries of words or link prediction in complex networks) compared to their Euclidean counterparts.
These results agree with the intuition that better accounting for the inherent geometry of the data can improve downstream predictions.

However, current literature is largely limited regarding methods for standard pattern recognition tasks such as classification and clustering for data points that lie in hyperbolic space.
Unless the task of interest calls for only rudimentary analysis of the embeddings such as calculating the (hyperbolic) distances or angles between pairs of data points, practitioners are limited to applying algorithms that are designed for data points in Euclidean spaces.
For example, when Chamberlain et al.~\cite{Chamberlain17} set out to classify nodes in a graph after embedding them into hyperbolic space, they resorted to performing logistic regression on the embedding coordinates, which relies on decision boundaries that are linear in the Euclidean sense, but are somewhat arbitrary when viewed in the underlying hyperbolic space.

To enable principled analytic pipelines where the inherent geometry of the data is respected in an \emph{end-to-end} fashion, we generalize linear support vector classifiers, which are one of the most widely-used methods for classification, to data points in hyperbolic space.
Despite the complexities of hyperbolic distance calculation, we prove that support vector classification in hyperbolic space can in fact be performed by solving a simple optimization problem that resembles the Euclidean formulation of SVM, elucidating the close connection between the two.
We experimentally demonstrate the superior performance of hyperbolic SVM over the Euclidean version on two types of simulated datasets (Gaussian point clouds and evolving scale-free networks) as well as real network datasets analyzed by Chamberlain et al.~\cite{Chamberlain17}.

The rest of the paper is organized as follows. We review hyperbolic geometry and support vector classification in Sections 1 and 2 and introduce our method, hyperbolic SVM, in Section 3. We provide experimental evaluations of hyperbolic SVM in Section 4, and conclude with discussion and future directions in Section 5.

%\begin{nolinenumbers}
\section{Review of Hyperbolic Space Models}
%\end{nolinenumbers}

While hyperbolic space cannot be isometrically embedded in Euclidean space, there are several useful models of hyperbolic geometry formulated as a subset of Euclidean space, each of which provides different insights into the properties of hyperbolic geometry~\cite{Anderson06}.
Our work makes use of three standard models of hyperbolic space---hyperboloid, Poincar\'{e} ball, and Poincar\'{e} half-space---as briefly described in the following. 

Consider an $(n+1)$-dimensional real-valued space $\Rb^{n+1}$, equipped with an inner product of the form
\begin{equation}
x*y=x_0y_0-x_1y_1-\cdots-x_ny_n.
\end{equation}
This is commonly known as the Minkowski space. The $n$-dimensional \emph{hyperboloid model} $\Lb^n$ sits inside $\Rb^{n+1}$ as the upper half (one of the two connected components) of a unit ``sphere'' with respect to the Minkowski inner product:
\begin{equation}
\Lb^n=\{x: x=(x_0,\dots,x_n)\in\Rb^{n+1}, x*x = 1, x_0>0 \}.
\end{equation}
The distance between two points in $\Lb^n$ is defined as the length of the geodesic path on the hyperboloid that connects the two points. It is known that every geodesic curve (i.e., hyperbolic line) in $\Lb^n$ is an intersection between $\Lb^n$ and a 2D plane that goes through the origin in the ambient Euclidean space $\Rb^{n+1}$, and vice versa.

Next, projecting each point of $\Lb^n$ onto the hyperplane $x_0=0$ using the rays emanating from $(-1,0,\dots,0)$ gives the \emph{Poincar\'{e} ball model}
\begin{equation}
\Bb^n=\{x:x=(x_1,\dots,x_n)\in\Rb^n, \|x\|^2<1\}
\end{equation}
where the correspondence to the hyperboloid model is given by
\begin{equation}
(x_0,\dots,x_n)\in \Lb^n \Leftrightarrow \left(\frac{x_1}{1+x_0},\dots,\frac{x_n}{1+x_0}\right)\in \Bb^n.
\end{equation}
Here, hyperbolic lines are either straight lines that go through the center of the ball or an inner arc of a Euclidean circle that intersects the boundary of the ball at right angles.

Another useful model of hyperbolic space is the \emph{Poincar\'{e} half-space model} 
\begin{equation}
\Hb^n=\{x:x=(x_1,\dots,x_n)\in\Rb^n,x_1>0\}
\end{equation}
which is obtained by taking the inversion of $\Bb^n$ with respect to a circle that has a radius twice that of $\Bb^n$ and is centered at a boundary point of $\Bb^n$.
If we center the inversion circle at $(-1,0,\dots,0)$, the resulting correspondence between $\Bb^n$ and $\Hb^n$ is given by 
\begin{equation}
(x_1,\dots,x_n)\in \Bb^n \Leftrightarrow \frac{1}{1+2x_1+\|x\|^2}\left(1-\|x\|^2,2x_2,\dots,2x_n\right)\in \Hb^n.
\end{equation}
In this model, hyperbolic lines are straight lines that are perpendicular to the boundary of $\Hb^n$ or Euclidean half-circles that are centered on the boundary of $\Hb^n$.

Existing methods for hyperbolic space embedding (e.g., ~\cite{AlanisLobato16,Chamberlain17,Nickel17}) output the embedding coordinates in their model of choice,
but as described above the coordinates of the corresponding points in other models can be easily calculated.

%\begin{nolinenumbers}
\section{Review of Support Vector Classification}
%\end{nolinenumbers}

Let $\{(x^{(j)},y^{(j)})\}_{j=1}^m$ be a set of $m$ training data instances, where the feature vector $x^{(j)}$ is a point in a metric space $\mathcal{X}$ with distance function $d$, and $y^{(j)}\in\{1,-1\}$ denotes the true label for all $j$. Let $h:\mathcal{X}\mapsto \{1,-1\}$ be any decision rule. The \textit{geometric margin} of $h$ with respect to a single data instance $(x,y)$ can be defined as
\begin{equation}
\gamma(h,(x,y)) = yh\cdot \inf \{ d(x',x) : x'\in \mathcal{X}, h(x')\neq h(x) \}.
\end{equation}
Intuitively, the geometric margin $\gamma$ measures how far one needs to travel from a given point to obtain a different classification. Note that $\gamma$ is a signed  quantity; it is positive for points where our prediction is correct (i.e., $y=h$) and negative otherwise. Increasing the value of  $\gamma$ across the training data points is desirable; for correct classifications, we increase our confidence, and for incorrect classifications, we minimize the error. 

Maximum margin learning of the optimal decision rule $h^\star$, which provides the foundation for support vector machines, can now be formalized as
\begin{equation} \label{eq:mm}
h^\star = \argmax_{h\in\mathcal{H}} \min_{j\in[m]} \gamma(h, (x^{(j)},y^{(j)})),
\end{equation}
where $\mathcal{H}$ is the set of candidate decision rules that we consider.

If we let the data space $\mathcal{X}$ be $\mathbb{R}^n$ and $d$ be the Euclidean distance function and consider only linear classifiers, i.e., $\mathcal{H}=\{h(x;w):w\in\mathbb{R}^n\}$ where
\begin{equation}\label{eq:linearh}
h(x;w)=\begin{cases}
1 & w^T x >0, \\
-1 & \text{otherwise,}
\end{cases}
\end{equation}
then it can be shown that the max-margin problem given in Eq.~\ref{eq:mm} becomes equivalent to solving the following convex optimization problem:
\begin{align}\label{eq:esvm}
\text{minimize}_{w\in\mathbb{R}^{n}} &\ \frac{1}{2}\|w\|^2 \\
\text{subject to} &\ y^{(j)}(w^Tx^{(j)}) \ge 1, \forall j\in[m] 
\end{align}
The resulting algorithm that solves this problem (via its dual) is known as support vector machines (SVM).
Introducing a relaxation for the separability constraints gives a more commonly used soft-margin variant of SVM
\begin{align}\label{eq:esvmsoft}
\text{minimize}_{w\in\mathbb{R}^{n}} &\ \frac{1}{2}\|w\|^2 + C \sum_{j=1}^m \max(0,1-y^{(j)}(w^Tx^{(j)}))
\end{align}
where the parameter $C>0$ determines the tradeoff between minimizing misclassification and maximizing the margin. Solving this optimization problem either in its primal form or via its dual has been established as a standard tool for classification in a wide range of domains~\cite{Fan08}. Note that in our formulation the bias parameter is implicitly handled by appending $1$ to the end of each $x$.

%\begin{nolinenumbers}
\section{Hyperbolic Support Vector Classification}
%\end{nolinenumbers}

We newly tackle the problem of solving the max-margin problem in Eq.~\ref{eq:mm} where the data points lie in hyperbolic space.
In particular, we will adopt the hyperboloid model to let $\mathcal{X}=\Lb^n$ and let $d$ be the hyperbolic distance function. Note that the data points need not be initially specified using the hyperboloid model, since coordinates in other models of hyperbolic space (e.g., Poincar\'{e} ball model) can be easily converted to $\Lb^n$.

Analogous to the Euclidean SVM, we consider a set of decision functions that lead to linear decision boundaries \emph{in the hyperbolic space}. 
It is known that any hyperbolic line (geodesic) in $\Lb^n$ is an intersection between $\Lb^n$ and a 2D Euclidean plane in the ambient space $\Rb^{n+1}$.
Thus, a natural way to define decision hyperplanes in $\Lb^n$ is to use $n$-dimensional hyperplanes in $\Rb^{n+1}$ as a proxy.
More precisely, we let 
\begin{equation}\label{eq:hyboH}
\mathcal{H}=\{h(x;w):w\in\Rb^{n+1},w*w<0\}
\end{equation}
where
\begin{equation}
h(x;w)=\begin{cases}
1 & w * x >0, \\
-1 & \text{otherwise,}
\end{cases}
\end{equation}
and $*$ denotes the Minkowski inner product. This decision function has the decision hyperplane $w*x=0$, which is an $n$-dimensional hyperplane in $\Rb^{n+1}$. Thus, the corresponding decision hyperplane in $\Lb^n$ obtained by its intersection with $w*x=0$ is the hyperbolic space-equivalent of a linear hyperplane, which can also be viewed as a union of geodesic curves.
We provide examples of linear decision hyperplanes in a two-dimensional hyperbolic space in Figure 1. For example for  $\Lb^2$, which can be visualized as an upper hyperboloid in 3D, our choice of $\mathcal{H}$ consists of every geodesic curve on the hyperboloid. Interestingly, our formulation obviates the need for a bias term as every hyperbolic hyperplane of codimension one is covered by our parametrization.

\begin{figure}
\centering
\includegraphics[trim={0 0.8cm 0 0.5cm},clip,scale=.9]{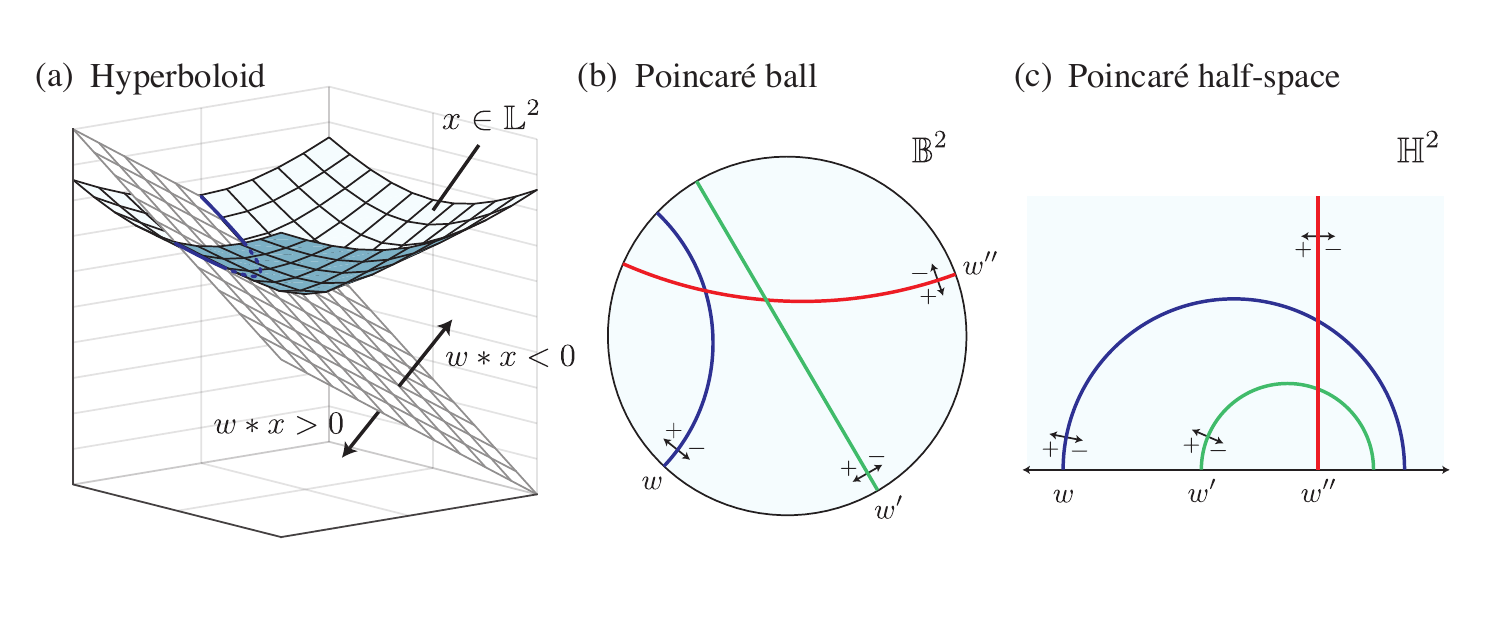}
\caption{\textbf{Linear decision hyperplanes in hyperbolic space models.} Examples where $w$, $w'$, and $w''$ denote different vectors in $\Rb^3$ that correspond to different decision hyperplanes in hyperbolic space. The correspondence between the hyperplanes in the two Poincar\'{e} models is meant as an illustration of concept and is not drawn to scale.}
\end{figure}

The condition that $w$ has negative Minkowski norm squared ($w*w<0$) is needed to ensure we obtain a non-trivial decision function; otherwise, the decision hyperplane does not intersect with $\Lb^n$  in $\Rb^{n+1}$ and thus all points in $\Lb^n$ are classified as the same label. 

Our first main result gives a simple closed-form expression for the geometric margin of a given data point to a decision hyperplane in hyperbolic space:
\begin{theorem}
Given $w\in\mathbb{R}^{n+1}$ such that $w*w<0$ and a data point $x\in\Lb^n$, the minimum hyperbolic distance from $x$ to the decision boundary associated with $w$, i.e., $\{z:w*z=0,z\in \Lb^n\}$, is given by 
\begin{equation}
\sinh^{-1}\left(\frac{w*x}{\sqrt{-w*w}}\right).
\end{equation}
\end{theorem}

To obtain this result, we reduce the problem of calculating the minimum hyperbolic distance to the decision hyperplane to a Euclidean geometry problem by mapping the decision boundary and the data point onto the Poincar\'{e} half-space model, in which the decision boundary is characterized as a Euclidean half-sphere.
A full proof of Theorem 1 is provided in Supplementary Information.

Given this formula, one can apply a sequence of transformations to the max-margin classification problem in Eq.~\ref{eq:hyboH} for the hyperbolic setting to obtain the following result.

\begin{theorem}
The maximum margin classification problem (Eq.~\ref{eq:mm}), with hyperbolic feature space $\mathcal{X}=L^n$ with (hyperbolic) distance function $d$,  and hyperbolic-linear decision functions $\mathcal{H}$ as defined in Eq.~\ref{eq:hyboH}, is equivalent to the following optimization problem:
\begin{align}\label{eq:hybomm}
\normalfont \text{minimize}_{w\in\mathbb{R}^{n+1}} &\ -\frac{1}{2}w*w, \\
\normalfont \text{subject to} &\ y^{(j)}(w*x^{(j)}) \ge 1, \forall j\in[m], \\
&\ w*w<0.
\end{align}
\end{theorem}

The proof of Theorem 2 is exactly analogous to the Euclidean version, and is provided in the Supplementary Information. 

Our result suggests that despite the apparent complexity of hyperbolic distance calculation, the optimal (linear) maximum margin classifiers in the hyperbolic space can be identified via a relatively simple optimization problem that closely resembles the Euclidean version of SVM, where Euclidean inner products are replaced with Minkowski inner products.

Note that if we restrict $\mathcal{H}$ to decision functions where $w_0=0$, then our formulation coincides with with Euclidean SVM. Thus, Euclidean SVM can be viewed as a special case of our formulation where the first coordinate (corresponds to the time axis in Minkowski spacetime) is neglected.

Unlike Euclidean SVM, however, our optimization problem has a non-convex objective as well as a non-convex constraint.
Yet, if we restrict our attention to non-trivial, finite-sized problems where it is necessary and sufficient to consider only the set of $w$ for which at least one data point lies on either side of the decision boundary, then the negative norm constraint can be replaced with a convex alternative that intuitively maps out the convex hull of given data points in the ambient Euclidean space of $\Lb^n$.

Finally, the soft-margin formulation of hyperbolic SVM can be derived by relaxing the separability constraints as in the Euclidean case.
Instead of imposing a linear penalty on misclassification errors, which has an intuitive interpretation as being proportional to the minimum Euclidean distance to the correct classification in the Euclidean case, we impose a penalty proportional to the \emph{hyperbolic} distance to the correct classification.
Analogous to the Euclidean case, we fix the scale of penalty so that the margin of the closest point to the decision boundary (that is correctly classified) is set to $\sinh^{-1}(1)$.
This leads to the optimization problem
\begin{align}\label{eq:hsvmsoft}
\text{minimize}_{w\in\mathbb{R}^{n+1}} &\ -\frac{1}{2}w*w + C\sum_{j=1}^m \max(0, \sinh^{-1}(1) - \sinh^{-1}(y^{(j)}(w*x^{(j)}))), \\
\text{subject to} &\ w*w<0.
\end{align}
In all our experiments in the following section, we consider the simplest approach of solving the above formulation of hyperbolic SVM via projected gradient descent.
The initial $w$ is determined based on the solution $w'$ of a soft-margin SVM in the ambient Euclidean space of the hyperboloid model, so that $w*x=(w')^Tx$ for all $x$.
This provides a good initialization for the optimization and has an additional benefit of improving the stability of the algorithm in the presence of potentially many local optima.

%\begin{nolinenumbers}
\section{Experimental Results}
%\end{nolinenumbers}

In the following, we compare hyperbolic SVM to the original Euclidean formulation of SVM (i.e., L2-regularized hinge-loss optimizer) on three different types of datasets---two simulated and one real.
For fair comparison, we restricted our attention to Euclidean SVM with a linear kernel, which has the same degrees of freedom as our method.
We discuss extending our work to non-linear classifications in hyperbolic space in Section 6.

%\begin{nolinenumbers}
\subsection{Evaluation Setting}
%\end{nolinenumbers}

To enable multi-class classification for datasets with more than two classes, we adopt a one-vs-all (OVA) strategy, where a collection of binary classifiers are independently trained to distinguish each class from the rest.
For each method, the resulting prediction scores on the holdout data are transformed into probability outputs via Platt scaling~\cite{Platt1999} across all classes and collectively analyzed to quantify the overall classification accuracy.
Note that for hyperbolic SVM we use the Minkowski inner product between the learned weight vector and the data point in the hyperboloid model as the prediction score, which is a monotonic transformation of the geometric margin.

In both hyperbolic and Euclidean SVMs, the tradeoff between minimizing misclassification and maximizing margin is determined by the parameter $C$ (see Eqs.~\ref{eq:esvmsoft} and \ref{eq:hsvmsoft}).
In all our experiments, we determined the optimal $C\in\{0.1, 1, 10\}$ separately for each run via a nested cross-validation procedure.

Our main performance metric is macro-averaged area under the precision recall curve (AUPR), which is obtained by computing the AUPR of predicting each class against the rest separately, then taking the average across all classes.
The results based on other performance metrics, such as the area under the ROC curve and the micro-average variants of both metrics, led to similar conclusions across all our experiments.

%\begin{nolinenumbers}
\subsection{Classifying Mixture of Gaussians}
%\end{nolinenumbers}

To evaluate hyperbolic SVM, we first generated a collection of 100 toy datasets by sampling data points from a Gaussian mixture model defined in the Poincar\'{e} disk model $\Bb^2$.
Note that, analogous to the Euclidean setting, the probability density function of an (isotropic) hyperbolic Gaussian distribution decays exponentially with the squared hyperbolic distance from the centroid, inversely scaled by the variance parameter.
For each dataset, we sampled four centroids from a zero-mean hyperbolic Gaussian distribution with variance parameter 1.5.
Then, we sampled 100 data points from a unit-variance hyperbolic Gaussian distribution centered at each of the four centroids to form a dataset of 400 points assigned to 4 classes.

The results from five independent trials of two-fold cross validation on each of the 100  datasets are summarized in Figure 2a. We observed a strongly significant improvement of hyperbolic SVM over the Euclidean version in terms of prediction accuracy, with a one-sided paired-sample $t$-test $p$-value $=6.17 \times 10^{-28}$.
 
We attribute the performance improvement of hyperbolic SVM to the fact that the class of decision functions it considers (linear hyperplanes in hyperbolic space) better match the underlying data distributions, which follow hyperbolic geometry by design.
The learned decision functions for both methods on an example dataset, generated in the same manner as above, are shown in Figures 2b and c.
Note that the apparent non-linearity of hyperbolic SVM decision boundaries are due to our use of the Poincar\'{e} disk model for visualization; in the hyperbolic space, these decision boundaries are in fact linear.

\begin{figure}
\centering
\includegraphics[trim={0 0.3cm 0 0.2cm},clip,scale=.9]{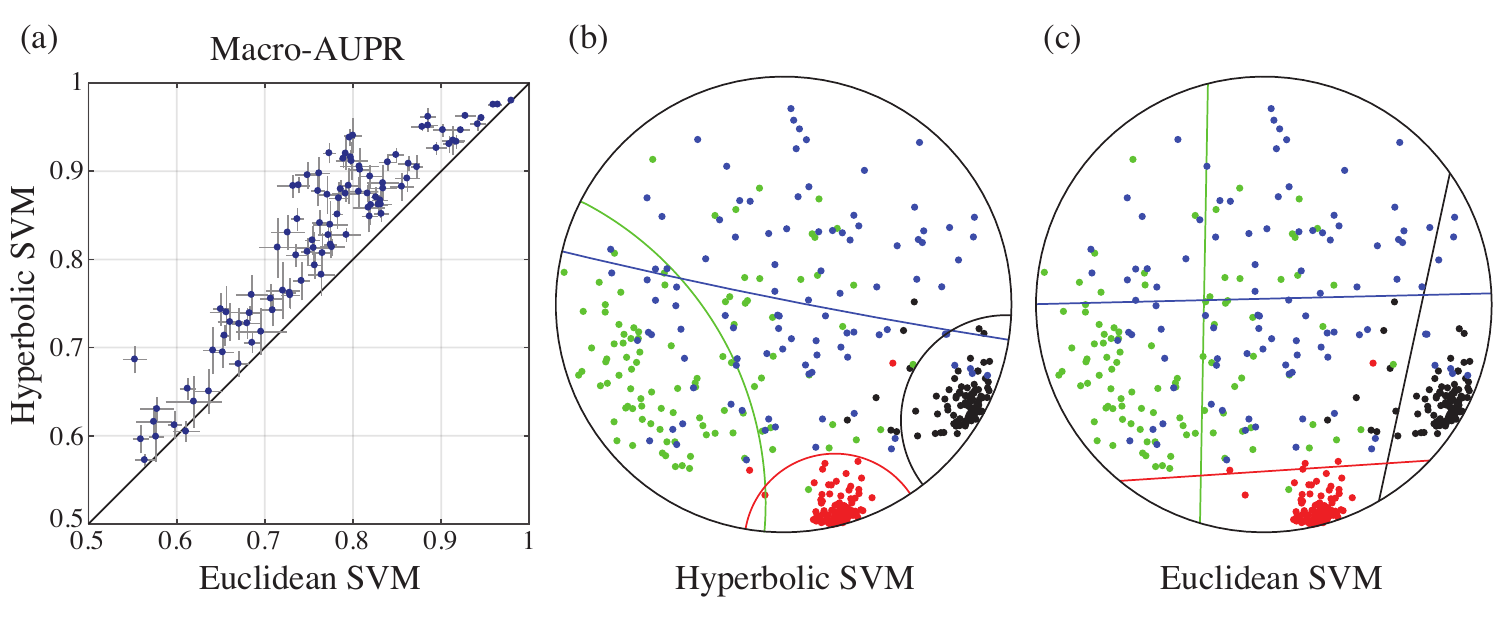}
\caption{\textbf{Multi-class classification of Gaussian mixtures in hyperbolic space}. (a) Two-fold cross validation results for 100 simulated Gaussian mixture datasets with 4 randomly positioned components and 100 points sampled from each component. Each dot represents the average performance over 5 trials. Vertical and horizontal lines represent standard deviations. Example decision hyperplanes for hyperbolic and Euclidean SVMs are shown in (b) and (c), respectively, using the Poincar\'{e} disk model. Color of each decision boundary denotes which component is being discriminated from the rest.}
\end{figure}

%\begin{nolinenumbers}
\subsection{Node Classification in Evolving Networks}
%\end{nolinenumbers}
One of the key applications of hyperbolic space embedding is the modeling of complex networks that exhibit scale-free properties~\cite{AlanisLobato16,Papadopoulos15,Papadopoulos12}, such as protein-protein interaction networks.
Here, we set out to test whether hyperbolic SVM can improve classification performance for the embedding of such networks.

To this end, we generated 10 random scale-free networks using the popularity-vs-similarity (PS) model, which was shown to faithfully reproduce the properties of networks in many real-world applications~\cite{Papadopoulos12} (Figure 3a).
The PS model starts with an empty graph and, at each time point, creates a new node and attaches it to a certain number of existing nodes where the likelihood of an edge depends on the degree of the node (popularity) as well as node-node similarities.
We embedded each of the simulated networks into hyperbolic space using LaBNE~\cite{AlanisLobato16} (Figure 3b), a network embedding method directly based on the PS model.

Inspired by the gene function prediction task in biology~\cite{Cho16}, we then generated a multi-class, multi-label dataset based on each of the simulated networks as follows.
For each new label, we randomly choose a node in the network to be the first ``pioneer'' node to be annotated with the label.
Then, we replay the evolution of the network, and every time a new node is connected to an existing node with the given label, the label was propagated to the new node with a set probability (0.8 in our experiments).
This procedure results in a relatively clustered set of nodes within the network being assigned to the same label.
Such patterns are prevalent in protein-protein interaction networks, where genes or proteins that belong to the same functional category tend to be proximal and share many connections in the network.

For each target size range for the label, where size refers to the final number of nodes annotated with the label, we created 10 such labels for each network to obtain a multi-label classification dataset with 10 classes. This process was repeated 5 times for each of the 10 networks to generate a total of 150 datasets with varying label sizes (20-50, 50-100, and 100-200). We evaluated the prediction performance of hyperbolic SVM via two-fold cross validation procedure, where we held out the labels of half the nodes in the network and predicted them based on the other half.

Our results are summarized in Figure 3c. Across all label size ranges and networks, hyperbolic SVM matched or outperformed Euclidean SVM. The overall improvement of hyperbolic SVM was statistically significant, with a one-sided paired-sample $t$-test $p$-value of $3.99\times 10^{-21}$. Notably, our performance improvement was more pronounced for smaller size ranges (20-50 and 50-100).
   
\begin{figure}
\centering
\includegraphics[trim={0 0.3cm 0 0.3cm},clip,scale=.9]{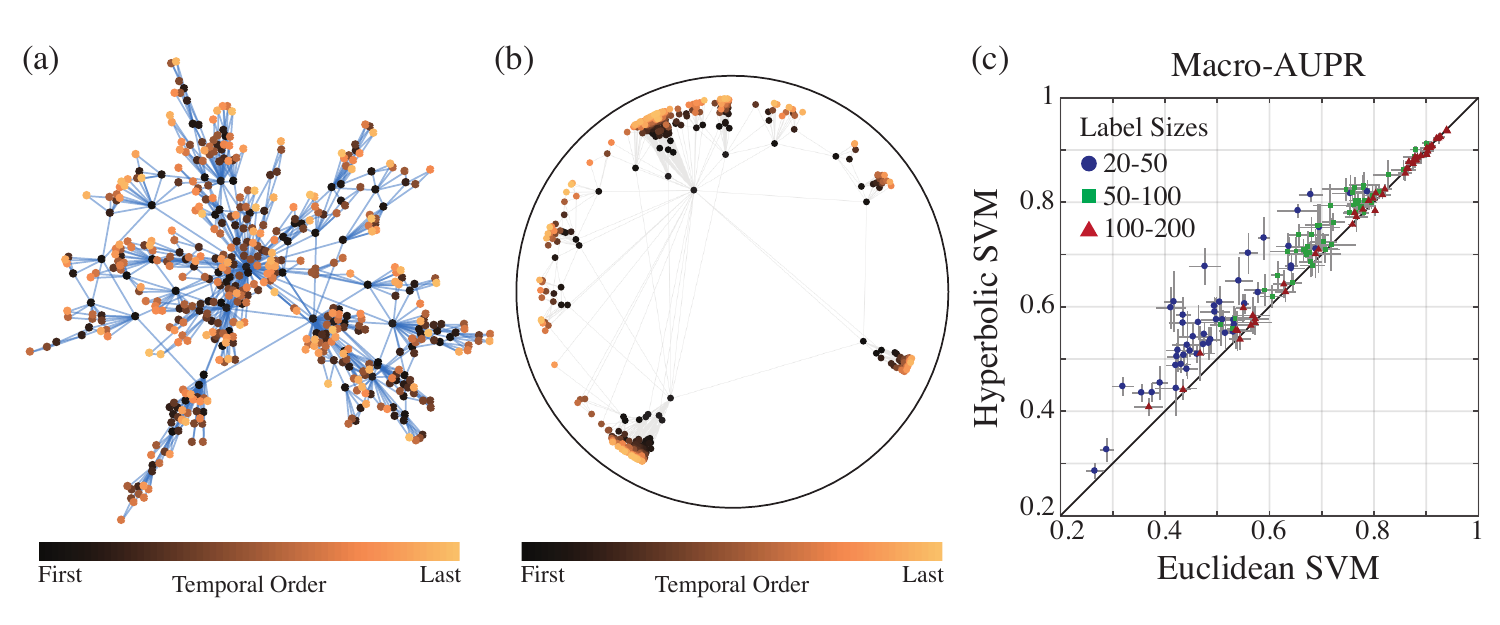}
\caption{\textbf{Multi-class multi-label classification of nodes in simulated evolving networks.} (a) One of the ten simulated networks used to generate our benchmark datasets, constructed using the PS model~\cite{Papadopoulos12} of network evolution. We set the number of nodes to $500$, average degree to $4$, scaling exponent to $2.25$, and temperature to $0$, to ensure sufficient clustering patterns in the network. (b) Embedding of the same network in two-dimensional hyperbolic space as visualized in the Poincar\'{e} disk model. Node classification in our experiment is performed using only the embedding coordinates, without access to the underlying network structure. (c) Two-fold cross validation results for predicting 10 labels per dataset, where each label is assigned to a random node at the time of its creation (in the PS model) and probabilistically propagated over newly added edges. Labels that were assigned to a number of nodes not in a pre-specified range were rejected and regenerated. We repeated the experiment for different size ranges for the labels, denoted by marker type/color. Each marker represents the average performance over 5 independent cross-validation trials. Vertical and horizontal lines represent standard deviations.
}
\end{figure}

%\begin{nolinenumbers}
\subsection{Node Classification in Real Networks}
%\end{nolinenumbers}

To demonstrate the performance of hyperbolic SVM on real-world datasets, we tested it on four network datasets used by Chamberlain et al.~\cite{Chamberlain17} for benchmarking their hyperbolic network embedding algorithm. These datasets include: (1) \emph{karate}~\cite{Zachary77}: a social network of 34 people divided into two factions, (2) \emph{polbooks}\footnote{http://www-personal.umich.edu/~mejn/netdata/}: co-purchasing patterns of 105 political books around the time of 2004 presidential election divided into 3 affiliations, (3) \emph{football}~\cite{Girvan02}: football matches among 115 Division IA colleges in Fall 2000 divided into 12 leagues, and (4) \emph{polblogs}~\cite{Lada05}: a hyperlink network of 1224 political blogs in 2005 divided into two affiliations.
Note that we excluded the \emph{adjnoun} dataset due the near-random performance of all methods we considered.

For each dataset, we embedded the network into a two-dimensional hyperbolic space using Chamberlain et al.'s embedding method based on random walks~\cite{Chamberlain17}, which closely follows an existing algorithm called DeepWalk~\cite{Perozzi14} except Euclidean inner products are replaced with a measure of hyperbolic angle.
Given the hyperbolic embedding of each network, we performed two-fold cross validation to compare the prediction accuracy of hyperbolic SVM with Euclidean SVM. Note that Chamberlain et al. performed logistic regression on the embedding coordinates, which can be viewed as a variant of Euclidean SVM with a different penalty function on the misclassification margins.

For all four datasets, hyperbolic SVM matched or outperformed the performance achieved by Euclidean SVM. Notably, the datasets where the performance was comparable between the two methods (karate and polblogs) consisted of only two well-separated classes, in which case a linear decision boundary is expected to achieve a reasonably high performance.

In addition, we tested Euclidean SVM based on the Euclidean embeddings obtained by DeepWalk with dimensions 2, 5, 10, and 25. 
Even with as many as 25 dimensions, Euclidean SVM was not able to achieve competitive prediction accuracy based on the Euclidean embeddings across all datasets.
This supports the conclusion that hyperbolic geometry likely underlies these networks and that increasing the number of dimensions for the Euclidean embedding does not necessarily lead to representations that are as informative as the hyperbolic embedding.

\begin{table}
\centering
\begin{tabular}{ |c|c|c|c|c|c|  }
 \hline
 \multirow{2}{*}{Classifier} & Embedding & \multicolumn{4}{c|}{Dataset} \\
\cline{3-6}
  & (Dimension) & karate & polbooks & football & polblogs \\
 \hline
 Hyperbolic SVM & Hyperbolic (2) & $\textbf{0.86} \pm 0.03$ & $\textbf{0.73} \pm 0.04$ & $\textbf{0.24} \pm 0.03$ & $\textbf{0.93} \pm 0.01$ \\ 
 Euclidean SVM & Hyperbolic (2) & $\textbf{0.86} \pm 0.03$ & $0.66 \pm 0.02$ & $0.21 \pm 0.01$ & $\textbf{0.93} \pm 0.01$ \\ 
 Euclidean SVM & Euclidean (2)  & $0.47 \pm 0.07$ & $0.34 \pm 0.03$ & $0.09 \pm 0.01$ & $0.60 \pm 0.09$ \\ 
 Euclidean SVM & Euclidean (5)  & $0.55 \pm 0.08$ & $0.35 \pm 0.03$ & $0.10 \pm 0.01$ & $0.69 \pm 0.04$ \\ 
 Euclidean SVM & Euclidean (10) & $0.50 \pm 0.08$ & $0.36 \pm 0.03$ & $0.10 \pm 0.01$ & $0.72 \pm 0.04$ \\ 
 Euclidean SVM & Euclidean (25) & $0.50 \pm 0.09$ & $0.37 \pm 0.04$ & $0.11 \pm 0.02$ & $0.80 \pm 0.03$ \\ 
 \hline
\end{tabular}

\ 
\caption{\textbf{Node classification performance on four real-world network datasets.} We performed two-fold cross validation experiments on four real-world network datasets described in main text. For all four datasets, hyperbolic SVM matched or outperformed the performance achieved by Euclidean SVM (the datasets where the performance of two methods were comparable, karate and polblogs, contained only two well-separated classes.) Methods are evaluated based on macro-averaged area under the precision recall curve. Mean performance summarized over 5 cross-validation trials over 5 different embeddings for each dataset is shown, each followed by the standard deviation. Numbers corresponding to the best performance on each dataset are shown in boldface.
}
\end{table}

%\begin{nolinenumbers}
\section{Discussion and Future Work}
%\end{nolinenumbers}

We proposed support vector classification in hyperbolic space and demonstrated its effectiveness in classifying points in hyperbolic space on three different types of datasets.
Although we focused on decision functions that are linear in hyperbolic space (i.e., based on hyperbolically geodesic decision hyperplanes),
our formulation of hyperbolic SVM may potentially allow the development of non-linear classifiers, drawing intuition from kernel methods for SVM.
In particular, we are interested in exploring the use of radial basis function kernels in hyperbolic space, which are widely-used in the Euclidean setting. In addition, while our experimental results were based on two-dimensional hyperbolic spaces, our formulation naturally extends to higher dimensional hyperbolic spaces, which may be of interest in future applications.

More broadly, our work belongs to a growing body of literature that aims to develop learning algorithms that directly operate over a Riemannian manifold~\cite{Porikli10,Tuzel08}. Linear hyperplane-based classifiers and clustering algorithms have previously been formulated for spherical spaces~\cite{Dhillon01,Lebanon04,Wilson10}. To the best of our knowledge, 
our work is the first to develop and experimentally demonstrate support vector classification in hyperbolic geometry. We envision further development of hyperbolic space-equivalents of other standard machine learning tools in the near future.

%\begin{nolinenumbers}

%\subsubsection*{Acknowledgments}
%Use unnumbered third level headings for the acknowledgments. All
%acknowledgments go at the end of the paper. Do not include
%acknowledgments in the anonymized submission, only in the final paper.
\newpage
%\section*{References}
\bibliographystyle{plain}
\bibliography{main}

\newpage
\section*{Proof of Theorem 1}

To derive margin-based classifiers in hyperbolic space, we first derive a closed-form expression for the geometric margin of a given data point $x\in L^n$ to a ``linear'' decision hyperplane defined by $w\in\mathbb{R}^{n+1}$ as we described in the previous section.

First, we perform an isometric transformation to simplify calculations. Let $A$ be an orthogonal matrix in $\mathbb{R}^n$. Then, the matrix 
\begin{equation*}
B=\begin{bmatrix}
1 & 0 \\
0 & A
\end{bmatrix}
\end{equation*}
represents an isometric, orthogonal transformation of the Minkowski space, since it preserves the associated inner product as follows
$$
(Bu) * (Bv) = u_0v_0 - u_{1:n}^TA^TAv_{1:n}= u_0v_0 - u_{1:n}^Tv_{1:n} = u*v
$$
for any $u$ and $v$.
If we set the first column of $A$ to $w_{1:n}/\|w_{1:n}\|$ where $\|\cdot\|$ denotes the Euclidean norm, then due to the preservation of geodesics under isometry, the margin of interest becomes equivalent to the margin of a transformed point 
$\tilde{x}=Bx$
to the decision hyperplane parameterized by
$$\tilde{w}=Bw=(w_0,\|w_{1:n}\|,0,\dots,0).$$
The first two coordinates of $\tilde{x}$ are given in terms of the original coordinates as 
$$\tilde{x}_0=x_0\text{ and }\tilde{x}_1=\frac{w_{1:n}^Tx_{1:n}}{\|w_{1:n}\|}.$$
Given such a transformation exists for any point in $L^n$, it is sufficient to derive the margin of an arbitrary point $\tilde{x}\in L^n$ with respect to a decision hyperplane represented by a weight vector $\tilde{w}$, where $$\tilde{w}_2=\cdots=\tilde{w}_n=0.$$
We use $\lambda$ to denote the ratio between the first two coordinates of $\tilde{w}$ as
$$
\lambda:=\frac{\tilde{w}_0}{\tilde{w}_1}.
$$
Importantly, the condition that $\tilde{w}*\tilde{w}<0$ in order for $\tilde{w}$ to represent a non-trivial decision function is equivalent to the condition that
$|\lambda| < 1$, which we will assume in our derivation.

The following lemma characterizes the decision hyperplane defined by such $\tilde{w}$:

\begin{lemma}
The decision hyperplane $D_{\tilde{w}}=\{x:\tilde{w}*x = 0,x\in L^n\}$ corresponding to a weight vector $\tilde{w}\in\mathbb{R}^{n+1}$ where $\tilde{w}_2=\cdots=\tilde{w}_n=0$ is equivalently represented in the Poincar\'{e} half-space model as a Euclidean hypersphere centered at the origin with radius $\sqrt{(1-\lambda)/(1+\lambda)}$, where $\lambda=\tilde{w}_0/\tilde{w}_1$.
\end{lemma}

\begin{proof}
It suffices to show for any $x\in\Lb^n$,
$$
x\in D_{\tilde{w}} \iff \| g(x) \|^2 = \frac{1-\lambda}{1+\lambda},
$$
where $g:\Lb^n\mapsto\Hb^n$ maps points in the hyperboloid model to the corresponding points in the half-space model and $\|\cdot\|$ denotes the Euclidean norm.

Let $h\in\Hb^n$, $b\in\Bb^n$, and $x\in\Lb^n$ be the points in the half-space model, the ball model, and the hyperboloid model, respectively, that represent the same point in the hyperbolic space. Since
$$
h=\frac{1}{1+2b_1+\|b\|^2} (1-\|b\|^2,2b_2,\dots,2b_n),
$$
we have
$$
\|h\|^2 = \frac{1}{(1+2b_1+\|b\|^2)^2} \left[ (1-\|b\|^2)^2 + \sum_{i=2}^n 4b_i^2 \right].
$$
Note that
\begin{align*}
(1-\|b\|^2)^2 + \sum_{i=2}^n 4b_i^2 &=(1-\|b\|^2)^2 + 4\|b\|^2 -4b_1^2 \\
&= (1+\|b\|^2)^2 -4b_1^2 \\
&= (1+2b_1+\|b\|^2)(1-2b_1+\|b\|^2),
\end{align*}
which gives us
$$
\|h\|^2 = \frac{1-2b_1+\|b\|^2}{1+2b_1+\|b\|^2}.
$$
Next, recall 
$$
b = \frac{1}{x_0+1}(x_1,\dots,x_n),
$$
which leads to
$$
\|b\|^2 =  \frac{x_1^2+\cdots+x_n^2}{(x_0+1)^2} = \frac{x_0^2-1}{(x_0+1)^2} = \frac{x_0-1}{x_0+1},
$$
where we used the fact that $x*x=1$ since $x\in\Lb^n$.

We can now express $\|h\|^2$ in terms of $x$ as
$$
\|h\|^2 = \frac{1-\frac{2x_1}{x_0+1}+\frac{x_0-1}{x_0+1}}{1+\frac{2x_1}{x_0+1}+\frac{x_0-1}{x_0+1}} = \frac{1-\frac{x_1}{x_0}}{1+\frac{x_1}{x_0}}.
$$
Because the function $f(z)=(1-z)/(1+z)$ is bijective,
$$
\|h\|^2 = \frac{1-\lambda}{1+\lambda} \iff \frac{x_1}{x_0} = \lambda = \frac{\tilde{w}_0}{\tilde{w}_1}.
$$
Finally, note that
\begin{align*}
\frac{x_1}{x_0} = \frac{\tilde{w}_0}{\tilde{w}_1} &\iff \tilde{w}_0x_0 - \tilde{w}_1x_1 =0 \\
&\iff {\tilde{w}}*x =0 \\
&\iff x\in D_{\tilde{w}},
\end{align*}
where we used the fact that $\tilde{w}_2=\cdots=\tilde{w}_n=0$.
\end{proof}

It is a known fact that, in a two-dimensional hyperbolic space, the set of points that are equidistant to a hyperbolic line on the same side of the line forms what is called a \emph{hypercycle}, which takes the shape of a Euclidean circle in the Poincar\'{e} half-plane model that goes through the same two \emph{ideal points} as the reference line. Note that the ideal points refer to the two end points of a hyperbolic line in the half-plane model (a circular arc representing a geodesic curve) where the hyperbolic line meets the boundary of the half-plane. 

In a high-dimensional setting, an analogous property is that the set of points equidistant to a hyperbolic hyperplane takes the shape of a Euclidean hypersphere that intersects the boundary of the Poincar\'{e} half-space at the same ideal points as the hyperplane.

Because our decision hyperplane as characterized in Lemma 1 has its center at the origin of the half-space model, any hypersphere that intersect the boundary of the half-space at the ideal points of the decision hyperplane must have a center $(c,0,\dots,0)$ for some $c\in\Rb$. In other words, any hypersphere representing a hypercycle with respect to our given decision boundary is centered on the first coordinate axis (which is perpendicular to the boundary of the half-space).

Let $\tilde{h}\in\Hb^n$ be the point in the half-space model that corresponds to the transformed data point $\tilde{x}$ we described earlier. One way to reason about the margin of $\tilde{x}$ with respect to the decision hyperplane defined by $\tilde{w}$ is to find which hypercycle $h$ belongs to. We can do so using the fact that, in addition to $\tilde{h}$, an ideal point of the decision hyperplane $(0,r_{\tilde{w}},0,\dots,0)$ lies on the hypercycle, where
$$
r_{\tilde{w}} := \sqrt{\frac{1-\lambda}{1+\lambda}}
$$
from Lemma~1.
More precisely, we can solve for the center of the hypercycle parameterized by $c$ using the equation
$$
(\tilde{h}_1-c)^2 + \sum_{i=2}^n \tilde{h}_i^2 = c^2 + r_{\tilde{w}}^2,
$$
which states that $\tilde{h}$ and $(0,r_{\tilde{w}},0,\dots,0)$ are equidistant from the center of the hypercycle $(c,0,\dots,0)$.
This gives us
$$
c = \frac{\| \tilde{h} \|^2 - r_{\tilde{w}}^2}{2\tilde{h}_1}.
$$
Now, using the relations
$$\|\tilde{h}\|^2 = \frac{\tilde{x}_0-\tilde{x}_1}{\tilde{x}_0+\tilde{x}_1}\text{ and }\tilde{h}_1=\frac{1}{\tilde{x}_0+\tilde{x}_1},$$
which can be derived based on the mapping between the hyperboloid and the half-space models, we obtain that
$$
c = \frac{(1-r_{\tilde{w}}^2)\tilde{x}_0 - (1+r_{\tilde{w}}^2)\tilde{x}_1}{2}.
$$

Next, we find a point $(\delta,0,\dots,0)\in\Hb^n$ that lies on the first coordinate axis and has the same margin with respect to the decision hyperplane as the given data point $\tilde{h}$.

Since we know that the radius of this hypercycle is given by $\sqrt{c^2+r_{\tilde{w}}^2}$ (i.e., the distance from $(c,0,\dots,0)$ to $(0,r_{\tilde{w}},0,\dots,0)$),
we get $$
\delta=c+\sqrt{c^2+r_{\tilde{w}}^2}.
$$

Finally, using the hyperbolic distance formula in the half-space model, we obtain that the hyperbolic distance between $(\delta,0,\dots,0)$ and $(r_{\tilde{w}},0,\dots,0)$, which is equal to the unsigned geometric margin of our data point, is given by the log-ratio
$$
\log \frac{\delta}{r_{\tilde{w}}}= \log \left( \frac{c}{r_{\tilde{w}}} + \sqrt{\left(\frac{c}{r_{\tilde{w}}}\right)^2+1}\right)= {\text{arsinh}}\left(\frac{c}{r_{\tilde{w}}}\right).
$$

Using the expression for $c$ we previously obtained, note that
$$
\frac{c}{r_{\tilde{w}}} = \frac{1}{2}\left[\left(\frac{1}{r_{\tilde{w}}}-r_{\tilde{w}}\right)\tilde{x}_0 - \left(\frac{1}{r_{\tilde{w}}}+r_{\tilde{w}}\right)\tilde{x}_1\right].
$$
Since 
$$
\frac{1}{r_{\tilde{w}}}-r_{\tilde{w}} = \sqrt{\frac{1+\lambda}{1-\lambda}}-\sqrt{\frac{1-\lambda}{1+\lambda}} = \frac{2\lambda}{\sqrt{1-\lambda^2}} = \frac{2\tilde{w}_0}{\sqrt{\tilde{w}_1^2 - \tilde{w}_0^2}},$$
and similarly,
$$
\frac{1}{r_{\tilde{w}}}+r_{\tilde{w}} = \frac{2}{\sqrt{1-\lambda^2}} =\frac{2\tilde{w}_1}{\sqrt{\tilde{w}_1^2 - \tilde{w}_0^2}},
$$
we can alternatively express the margin as
$$
{\text{arsinh}} \left( \frac{\tilde{w}_0 \tilde{x}_0 - \tilde{w}_1 \tilde{x}_1}{\sqrt{\tilde{w}_1^2 - \tilde{w}_0^2}} \right) = {\text{arsinh}} \left( \frac{\tilde{w} *\tilde{x}}{\sqrt{-\tilde{w} * \tilde{w}}} \right),
$$
using the fact that $\tilde{w}_2=\cdots=\tilde{w}_n=0$.

Finally, since we have shown earlier that our initial transformation of $w$ and $x$ to $\tilde{w}$ and $\tilde{x}$ preserves the Minkowski inner product, we obtain the geometric margin in terms of the original variables as
$$
{\text{arsinh}} \left( \frac{w*x}{\sqrt{-w * w}} \right).
$$

\section*{Proof of Theorem 2}

Note that if we make the rescaling $w\to \kappa w$, the distance of any point to the decision surface is unchanged. Indeed, using Theorem 1, and properties of inner products,
$$\sinh^{-1}\left(\frac{\kappa w* x}{\sqrt{-\kappa w* \kappa w}}\right)
=\sinh^{-1}\left(\frac{\kappa(w* x)}{\kappa\sqrt{-w* w}}\right)
=\sinh^{-1}\left(\frac{(w* x)}{\sqrt{-w* w}}\right)$$
so scaling by $\kappa$ does not matter. Using this freedom, we can assume that $y^{0}(w* \textbf{x}^0)=1$, where $\textbf{x}^0$ is the point with minimum hyperbolic distance to the decision plane specified by $w$, and $y^{0}\in \{-1,1\}$ is the corresponding decision. This is the analogue of the ``canonical representation" of decision hyperplanes familiar from Euclidean SVMs. When $w$ is thus scaled, we have
$$y^{(j)}(w*x^{(j)})\geq 1$$
for all $j\in [m]$. 

Therefore, the optimization problem now simply maximizes 
$$\sinh^{-1}\left(\frac{1}{\sqrt{-w*w}}\right)$$
which is equivalent to minimizing $-\frac{1}{2}w*w$ subject to the above constraint. We add the factor of $\frac{1}{2}$ to simplify  gradient calculations. The additional constraint $w*w<0$ ensures that the decision function specified by $w$ is nontrivial (see main text).

%\end{nolinenumbers}

\end{document}